\documentclass[letterpaper]{article}
\usepackage{aaai}
\usepackage{times}
\usepackage{helvet}
\usepackage{courier}
\usepackage{amsmath}
\usepackage{amssymb}
\usepackage{amsthm}
\usepackage{booktabs}
\usepackage{nameref}
\usepackage{tikz}
\usepackage{graphicx}
\usepackage{listings}
\usepackage{array}
\usepackage{url}
\usepackage{xspace}
\usetikzlibrary{shapes}
\newtheorem{thm}{Theorem}[section]
\newtheorem{lemma}[thm]{Lemma}
\newcommand{\algo}{CombU\xspace}
 
\begin{document}
%
\title{CombU: A Combined Unit Activation for Fitting Mathematical Expressions with Neural Networks}
\author{Jiayu Li, Zilong Zhao\thanks{Contact: li.jiayu@nus.edu.sg, z.zhao@nus.edu.sg}\\
Betterdata AI\\
NUS AIDF\\
Singapore, Singapore\\
\And
Kevin Yee, Uzair Javaid\\
Betterdata AI \\
Singapore, Singapore
\And
Biplab Sikdar\\
NUS AIDF\\
NUS Department of Electrical \\and Computer Engineering \\
Singapore, Singapore
}
\maketitle
\begin{abstract}
\begin{quote}


The activation functions are fundamental to neural networks as they introduce non-linearity into data relationships, thereby enabling deep networks to approximate complex data relations. Existing efforts to enhance neural network performance have predominantly focused on developing new mathematical functions. However, we find that a well-designed combination of existing activation functions within a neural network can also achieve this objective. In this paper, we introduce the \textbf{Comb}ined \textbf{U}nits activation (\algo), which employs different activation functions at various dimensions across different layers. This approach can be theoretically proven to fit most mathematical expressions accurately. The experiments conducted on four mathematical expression datasets, compared against six State-Of-The-Art (SOTA) activation function algorithms, demonstrate that \algo outperforms all SOTA algorithms in 10 out of 16 metrics and ranks in the top three for the remaining six metrics. 

\end{quote}
\end{abstract}

\section{Introduction}


In deep learning, an activation function is a mathematical function applied to the output of a neural network's node. It plays a critical role in determining whether a neuron should be activated, based on the weighted sum of its inputs, thereby introducing non-linearities into the model. This non-linearity is essential as it allows the neural network to learn complex patterns and representations from the data, which would be impossible with linear functions alone.
Since the earliest usage of Sigmoid activations~\cite{sigmoid}, and the revolutionary discovery of the effectivity and efficiency of Rectified Linear Unit (ReLU)~\cite{relu}, many different further optimized activation functions for different tasks are proposed, such as Tanh~\cite{tanh}, Leaky ReLU~\cite{lrelu}, ELU~\cite{elu}, SELU~\cite{selu}, NLReLU~\cite{nlrelu}, and GELU~\cite{gelu}. 

Most existing attempts have been done to change the activation function universally in a network, while very few attempts, if any, has been done to explore the capability of combining different activations in the same network. However, carefully designed combination may have the potential of making an effect of $1+1>2$.

In particular, using ReLU throughout an entire network essentially fits a piece-wise linear function, with sufficiently large number of pieces, may approximate real-world scenarios well enough empirically. Using other linear unit kind of activation functions also results in piece-wise functions but some pieces may be small curves. None of the activation functions, if used universally throughout an entire network, can provide a perfect fit of very simple and ubiquitous data relations expressed in mathematics, including quadratic, exponential and logarithmic relations and compositions of them. However, fitting these simple mathematical expressions is not out of the capability of neural networks if the activation functions can be carefully designed and combined.

In this paper, we propose \textbf{Comb}ined \textbf{U}nits activation (CombU), that applies different activation functions
at different dimensions of different layers. This novel combination is proven to have the potential of perfectly fitting common mathematical expressions in theory, when it effectively combines activations that has linear, exponential, and logarithmic components. In the following sections, we will prove this claim theoretically, and formally state the design of CombU. Experiments are done to verify the capability of CombU on fitting existing complex mathematical expressions. Results show that it outperforms SOTA activation functions 10 out of 16 metrics on understanding mathematical expressions. 
Experiments also demonstrate CombU's general advantage as multilayer perception (MLP) activations in tabular tasks, including classification, regression, and generation, which are believed to embed some implicit mathematical relations. To summarize, our contributions are as follows:

\begin{itemize} 
    \item We propose a new activation function - \algo that mixes ReLU, ELU, and NLReLU activation functions in all intermediate activation network layers. 

    \item Theoretical proof of \algo is provided to show that it can perfectly fit common mathematical expressions.

    \item Network with \algo is tested on both prediction tasks (classification and regression) and generation tasks. \algo is compared to 6 SOTA algorithms across 14 datasets. Results show that, on average, \algo outperformed all six SOTA algorithms on prediction tasks involving mathematical expressions and classification datasets, as well as on generation tasks involving two tabular datasets. For the regression dataset, \algo achieves the best performance on one metric and ranks second on the remaining metric.
\end{itemize}


\section{Related Work}

The very first neural networks used Sigmoid function for activation~\cite{sigmoid}, which is bounded in $[0,1]$, for doing binary decisions. Later, it found that Rectified Linear Unit (ReLU)~\cite{relu}, which is non-smooth and non-probabilistic yet easier and more efficient, shows great advantage on deep neural networks by performing as some decision gates. This discovery has been the foundation of modern neural networks and most subsequent research efforts on activation functions.

Most effort has been done to propose new unary activation functions, including SoftPlus~\cite{relu}, Tanh~\cite{tanh}, Leaky Rectified Linear Unit (LReLU)~\cite{lrelu}, Exponential Linear Unit (ELU)~\cite{elu}, Scaled Exponential Linear Unit (SELU)~\cite{selu}, Sigmoid Linear Unit (SiLU, also known as Swish)~\cite{swish}, Natural Logarithmic Rectified Linear Unit (NLReLU)~\cite{nlrelu}, and Gaussian Error Linear Unit (GELU)~\cite{gelu}. We summarize the activation functions mentioned above in Table~\ref{tab:related-work}.

\begin{table}[h!t]
    \centering
    \Large
    \setlength{\extrarowheight}{5pt}
    \resizebox{1\columnwidth}{!}{
    \begin{tabular}{llll}
        \toprule
        Name & Function ($f(x)$) & Derivative ($f'(x)$) & Range \\
        \midrule
        Sigmoid & \scalebox{1.5}{$\frac{1}{1+e^{-x}}$} & $f(x)\cdot(1-f(x))$ & $(0,1)$ \\[5pt]\hline
        ReLU & $\max(0,x)$ & $\begin{cases}
            1&\text{ if }x>0\\0&\text{ if }x<0
        \end{cases}$ & $[0,+\infty)$ \\[5pt]\hline
        SoftPlus & $\ln(1+e^x)$ & \scalebox{1.5}{$\frac{1}{1+e^{-x}}$} & $(0,+\infty)$\\[5pt]\hline
        Tanh & \scalebox{1.5}{$\frac{e^x-e^{-x}}{e^x+e^{-x}}$} & $1-f(x)^2$ & $(-1,1)$ \\[5pt]\hline
        LReLU ($a\in[0,1])$ & $\begin{cases}
            x&\text{ if }x\ge0\\ax&\text{ if }x<0
        \end{cases}$ & $\begin{cases}
            1&\text{ if }x>0\\a&\text{ if }x<0
        \end{cases}$ & $(-\infty,+\infty)$ \\[5pt]\hline
        ELU ($\alpha\in[0,1])$& $\begin{cases}
            x&\text{ if }x>0\\\alpha(e^x-1)&\text{ if }x\le0
        \end{cases}$ & $\begin{cases}
            1&\text{ if }x>0\\f(x)+\alpha&\text{ if }x<0
        \end{cases}$ & $(-\alpha,+\infty)$\\[5pt]\hline
        SELU ($\lambda>1$)& $\lambda\cdot\text{ELU}(x)$ & $\lambda\cdot\text{ELU}'(x)$ & $(-\lambda\alpha,+\infty)$  \\[5pt]\hline
        Swish ($\beta\ge0$) & \scalebox{1.5}{$\frac{x}{1+e^{-\beta x}}$} & \scalebox{1.5}{$\frac{1+e^{-\beta x}+\beta xe^{-\beta x}}{(1+e^{-\beta x})^2}$} & $[\delta,+\infty)$ \\[5pt]\hline
        NLReLU ($\beta>0$) & $\ln(\beta\cdot\max(0,x)+1)$ & $\begin{cases}
            \frac{\beta}{\beta x+1}&\text{ if }x>0\\0&\text{ if }x<0
        \end{cases}$ & $[0,\infty)$\\[5pt]\hline
        GELU & $x\Phi(x)$ & $\Phi(x)+x\phi(x)$ & $[\delta,+\infty)$ \\
        \bottomrule
    \end{tabular}
    }
    \caption{Existing activation functions. Activation functions are denoted $f(x)$. Any other unknowns than $x$ are hyper-parameters of the activation functions, whose range is stated with name. Function $\phi$ and $\Phi$ denote the probability density and cumulative density functions of Gaussian distributions respectively. $\delta$ is a constant value for the activation function, which is the solution for $x$ such that $f'(x)=0$.}
    \label{tab:related-work}
\end{table}

There is a special group of activation functions that are parameterized, with tune-able parameters, such as Parametric Rectified Linear Unit (PReLU)~\cite{prelu} that tunes $a$ in LReLU, Parametric Exponential Linear Unit (PELU)~\cite{pelu} that tunes $\alpha$ in ELU. 
There has also been work on exploring different activation functions. For example, the work of Swish~\cite{swish} actually uses reinforcement learning techniques to search for a best activation function. However, no work has been done to put different activations in different dimensions in the same layer across different layers in the same neural network.

Note that in this paper, we focus on \textbf{fitting} mathematical expressions, rather than \textbf{understanding} them in natural language, so that works on language models for mathematical reasoning do not fall under the realm of discussion of this paper.

\section{CombU Explained}

\subsection{CombU Formalization}

The proposed method is not an activation function per se, but a strategy of combining different activation functions. The implementation is straightforward, where each internal activation layer's dimensions are assigned with different activation functions based on a designated proportion. For example, with a CombU with 0.5 of activation A, 0.25 of activation B, and 0.25 of activation C, then given an activation layer on input data with $D$ dimensions, $0.5D$ randomly selected dimensions will be activated by A, another $0.25D$ other randomly selected dimensions will be activated by B, and the rest $0.25D$ dimensions will be activated by C. Some rounding may be applied if the dimensions are not integers. 

Formally, suppose the input is a vector $\mathbf{x}\in\mathbb{R}^D$, and activation functions considered are $\mathbf{g}=[g_1,g_2,\dots,g_k]$. Before training starts, what activation functions need to be used should be determined. It can be expressed as a series of masks $\mathbf{M}=[\mathbf{m}_1,\mathbf{m}_2,\dots,\mathbf{m}_k]\in\{0,1\}^{k\times D}$, with the constraint that $\sum_{i=1}^k\mathbf{m}_i=\mathbf{1}=[1,1,\dots,1]$. Then, CombU can be expressed as Equation~\ref{eq:combu}.

\begin{equation}
\label{eq:combu}
    \text{CombU}_{\mathbf{g},\mathbf{M}}(\mathbf{x})=\sum_{i=1}^k\mathbf{m}_i^Tg_i(\mathbf{x})
\end{equation}

\subsection{CombU Motivation}

Almost all existing neural-network-based research works, from simple MLPs to complex Large Language Models (LLMs), use a single activation function throughout all dimensions of all layers (excluding the last layer). However, we contend that using different activation functions combined has some benefits. 

The intuition is that neural networks using ReLU~\cite{relu} or its variants, such as Leaky ReLU~\cite{lrelu} and PReLU~\cite{prelu}, theoretically approximates the result by piece-wise linear functions. However, when domain experts develop glass-box methods using advanced empirical and analytical approaches, the derived theoretical or estimated functions are hardly ever piece-wise linear functions with a very high number of pieces, but rather, more complex functions involving polynomials, quotients, exponents, and logarithms. Despite the practical powerfulness of deep neural networks, simple mathematical expressions made up of these non-linear functions can never be perfectly fitted using a neural network with a uniform activation function.

Nevertheless, deep neural networks \textbf{\textit{can}} be used to perfectly fit these common mathematical expressions if different activation functions can be applied to different dimensions and at different layers. As long as different units in the same layer always include some linear, some exponential, and some logarithmic components, all the mathematical expressions mentioned above can be expressed by a neural network perfectly. 



\subsection{Combination Choice and Theoretical Analysis}
\label{sec:theo:proof}


In light of the requirement of existence of all linear, exponential, and logarithmic components in the activation of each layer, we select a CombU with 0.5 of ReLU~\cite{relu} for linear component, 0.25 of ELU~\cite{elu} (with $\alpha=1$) for linear and exponential component, and 0.25 of NLReLU~\cite{nlrelu} (with $\beta=1$) for linear and logarithmic component, as the default choice of CombU. 

A sketch of the proof is that exponential and logarithmic components are readily captured by existence of ELU and NLReLU, and polynomials and quotients can be captured by exponential after linear combination of logarithmic, for example, $x_1^ax_2^b=\exp(a\ln x_1+b\ln x_2)$, where $a,b>0$ gives a polynomial representation, and $a>0,b<0$ gives a quotient representation. 

We present a rigorous constructive proof below. Given that this is a theoretical analysis, 
we assume there are no limitations on the number of layers and the size of hidden layers, as they can always be set to be sufficiently large. Several basic and straightforward lemmas are assumed without proof, as their rigorous proofs can be found in \nameref{app:proof}.

We use $L$ to represent the total number of layers in the MLP. The variable $\mathbf{x}^{[l]}$ represents the input at layer $l,l\in\{1,\dots,L\}$ before activation (where the MLP output before the final activation layer, such as softmax or otherwise, is $\mathbf{x}^{[L]}$). The variable $\mathbf{z}^{[l]},l\in\{0,\dots,L-1\}$ represents the output after activation (the raw input is thus $\mathbf{x}=\mathbf{z}^{[0]}$). Let the size of each layer be $D_l,l\in\{1,2,\dots,L\}$. Let $\boldsymbol{\Theta}$ represent all trainable parameters of the model, where $\boldsymbol{\theta}^{[l]}$ represents the parameters for layer $l$, and $\boldsymbol{\theta}^{[l]}_{ij}$ is the weight from $y^{[l-1]}_i$ to $x^{[l]}_j$, and $i=0$ refers to the bias. To write expressions succinctly, we use $R$ to represent ReLU, $E$ to represent ELU (with $\alpha=1)$, $N$ to represent NLReLU (with $\beta=1$). 
In this section, we focus on the MLP without the final activation layer, which is task-specific (e.g., softmax for classification tasks). 
We also assume a finite dataset size, allowing us to empirically establish a range for $\mathbf{x}$, such that $\mathbf{x}\in(-M,M)^{D_0}$ for some sufficiently large $M\in\mathbb{R}$. Additionally, we define a minimum non-zero absolute value threshold, ensuring that $\min_{x\in\mathbf{x}\setminus\{0\}}(|x|)>\delta$.
Thus, the domain for each value is $\mathcal{D}=((-M,M)\setminus(-\delta,\delta))\cup\{0\}$.

\begin{lemma}
    \textbf{Capability for Exponential Expressions.} Exponential function can be constructed by a network. Formally writing, given $x\in\mathcal{D}, \exp(x)$ can be constructed as an output of a network. 
\end{lemma}

\begin{proof}
    \begin{align}
        &E(x-M)-R(x-M)\\
        =&\begin{cases}
            (x-M)-(x-M)&\text{ if }x-M\ge0\\
            \exp(x-M)-1-0&\text{ if }x-M<0
        \end{cases}\\
        =&\exp(x-M)-1
    \end{align}

    by definition of $M$. Then,

    \begin{align}
        &(E(x-M)-R(x-M)+1)\cdot\exp(M)\\
        =&\exp(x-M)\cdot\exp(M)=\exp(x)
    \end{align}

    Therefore, suppose $x$ is the only input dimension, $\exp(x)$ can be constructed with $\theta^{[1]}_{0\cdot}=-M, \theta^{[1]}_{1\cdot}=1,g^{[1]}_1=E,g^{[1]}_2=R,\theta^{[2]}_{11}=\exp(M),\theta^{[2]}_{21}=\theta^{[2]}_{01}=-\exp(M)$.
\end{proof}

\begin{lemma}
    \textbf{Capability for Logarithmic Expressions.} Logarithmic function can be constructed by a network. Formally writing, given $x\in\mathcal{D}, \ln(x)$ can be constructed as an output of a network. 
\end{lemma}

\begin{proof}
    When $\beta=1$,
    \begin{align}
        \ln\delta\cdot N(\frac{x}{\delta}-1)&=\begin{cases}
            \ln\delta\cdot\ln(\frac{x}{\delta})&\text{ if }x\ge\delta\\
            0&\text{ if }x<\delta
        \end{cases}\\
        &=\begin{cases}
            \ln(x)&\text{ if }x\ge0\\
            0&\text{ if }x<0
        \end{cases}
    \end{align}

    by definition of $\delta$. 

    Therefore, suppose $x$ is the only input dimension, $\ln(x)$ can be constructed with $\theta^{[1]}_{01}=-1,\theta^{[1]}_{11}=\frac{1}{\delta},g_1^{[1]}=N,\theta^{[2]}_{01}=0,\theta^{[2]}_{11}=\ln\delta$.
\end{proof}

\begin{thm}
    \textbf{Capability for Polynomials.} Polynomials can be constructed by a network. Here by polynomials we do not restrict ourselves to integer exponents. Formally writing, $\forall K\in\mathbb{N}_+,a_i,p_{ij}\in\mathbb{R}$, given $\mathbf{x}\in\mathcal{D}^D,\sum_{i=1}^Ka_i\prod_{j=1}^Dx_j^{p_{ij}}$ can be constructed by a network.
\end{thm}

\begin{proof}
    \begin{align}
        &\sum_{i=1}^Ka_i\prod_{j=1}^Dx_j^{p_{ij}}\\
        =&\sum_{i=1}^Ka_i\exp\left(\sum_{j=1}^Dp_{ij}\ln(x)\right)
    \end{align}

    Exponential and logarithmic can be constructed, as discussed previously. Linear combinations clearly can be constructed by Law of Linear Combinations in Next Layer and Law of Composition (see \nameref{app:proof}). Thus, by applying Law of Composition again, the polynomials can be constructed.
\end{proof}

In Capability for Polynomials, when the polynomial degrees are negative, they are essentially quotient expressions. Further, Law of Composition can be re-applied on it, 
so that polynomials of express-able expressions can also be expressed.
Therefore, all mathematical expressions involving only polynomial, quotient, exponential, and/or logarithmic components can be constructed by a network using the chosen CombU.

\section{Experimental Results and Discussion}

We compare the performance of CombU with ReLU~\cite{relu}, ELU~\cite{elu}, SELU~\cite{selu}, Swish~\cite{swish}, NLReLU~\cite{nlrelu}, GELU~\cite{gelu} in 
experiments to fit mathematical expressions and tabular data.
In this section, we only show the core experiments and data. Supplementary information and experiments can be found in \nameref{app:sup}.

\subsection{Fitting Mathematical Formulae}

In the previous section, we have justified that using CombU has the potential to fit mathematical expressions perfectly. Therefore, in this section, we first evaluate whether using CombU indeed enhances the performance of MLPs in fitting mathematical formulas.

We selected 
4 complex mathematical formulae from different expertise domains, including Gaussian distribution (GS)~\cite{stats} in statistics, modified Arrhenius equation (AR)~\cite{ar-math} in chemistry, 3D steady-state vortex solution of Naiver-Stokes equation (NS)~\cite{ns-math} in physics, and Black–Scholes formula for prices (BS)~\cite{bs-math} in finance. These formulae involve a variety of the mathematical expressions mentioned above, including typical polynomials, quotients, square roots, exponentials, logarithms, and their compositions. For each expression, a regression task is naturally defined as each involve a definite relation between the dependent variable and independent variables. By binning the dependent variable into a few classes, we also define a classification task. 

A summary of mathematical expressions used is given in Table~\ref{tab:math-summary}. Detailed data and training setup can be found in \nameref{app:exp}. 

\begin{table}[h!t]
    \centering
    \begin{tabular}{ccccccc}
        \toprule
        Abbr. & Field & HP & QT & SQ & EX & LG \\
        \midrule
        GS & Statistics & $\checkmark$ & $\checkmark$ & & $\checkmark$ \\
        AR & Chemistry & $\checkmark$ & $\checkmark$ & & $\checkmark$ \\
        NS & Physics & $\checkmark$ & $\checkmark$ & $\checkmark$ \\
        BS & Finance & $\checkmark$ & $\checkmark$ & $\checkmark$ & $\checkmark$ & $\checkmark$ \\
        \bottomrule
    \end{tabular}
    \caption{Summary of mathematical expressions tested. HP stands for higher order (more than one degree, namely, non-linear, such as $x_1^2$ and $x_1x_2$) polynomial, QT stands for quotient (in denominator), SQ stands for square root, EX stands for exponential, LG stands for logarithm, and ticks mean that the expression contains terms involving the relevant operations, including their compositions (so their input may not directly be features, but other expressions). }
    \label{tab:math-summary}
\end{table}

\begin{table*}[h!t]
    \small
    \setlength{\tabcolsep}{3pt}
    \centering
    \begin{tabular}{ccccccccc}
        \toprule
        \multicolumn{2}{c}{Expression} & ReLU & ELU & SELU & Swish & NLReLU & GELU & CombU\\
        \midrule
         & MAE & $0.088\pm0.011$ & $0.114\pm0.013$ & $0.138\pm0.006$ & $0.168\pm0.005$ & $0.078\pm0.005$ & $0.147\pm0.008$ & $\mathbf{0.073\pm0.005}$\\
        GS & MSE & $0.017\pm0.003$ & $0.036\pm0.008$ & $0.053\pm0.007$ & $0.083\pm0.006$ & $0.015\pm0.001$ & $0.062\pm0.006$ & $\mathbf{0.014\pm0.002}$\\
         & Acc. & $87.30\pm1.14$ & $87.36\pm0.32$ & $81.06\pm0.56$ & $88.52\pm0.76$ & $86.54\pm1.32$ & $\mathbf{90.82\pm0.44}$ & $89.44\pm0.92$ \\
         & F1 & $87.42\pm1.16$ & $87.57\pm0.33$ & $81.45\pm0.50$ & $88.73\pm0.74$ & $86.70\pm1.28$ & $\mathbf{90.98\pm0.42}$ & $89.58\pm0.92$ \\\midrule
         & MAE & $0.029\pm0.004$ & $0.067\pm0.013$ & $0.050\pm0.017$ & $0.045\pm0.004$ & $0.041\pm0.012$ & $0.032\pm0.004$ & $\mathbf{0.027\pm0.003}$ \\
        AR & MSE & $0.081\pm0.035$ & $0.070\pm0.026$ & $0.074\pm0.025$ & $0.097\pm0.021$ & $0.076\pm0.034$ & $\mathbf{0.063\pm0.019}$ & $\mathbf{0.064\pm0.021}$\\
         & Acc. & $88.04\pm0.91$ & $87.56\pm1.08$ & $88.54\pm0.70$ & $86.90\pm0.33$ & $\mathbf{90.38\pm0.58}$ & $87.16\pm0.56$ & $88.86\pm1.35$\\
         & F1 & $88.08\pm0.90$ & $87.62\pm1.06$ & $88.59\pm0.73$ & $87.01\pm0.32$ & $\mathbf{90.41\pm0.59}$ & $87.25\pm0.56$ & $88.91\pm1.36$\\\midrule
         & MAE & $\mathbf{0.106\pm0.004}$ & $0.162\pm0.006$ & $0.155\pm0.007$ & $0.157\pm0.006$ & $0.111\pm0.006$ & $0.131\pm0.009$ & $0.111\pm0.005$ \\
        NS & MSE & $2.420\pm0.268$ & $2.977\pm0.029$ & $2.889\pm0.028$ & $2.974\pm0.013$ & $2.525\pm0.165$ & $2.771\pm0.039$ & $\mathbf{2.414\pm0.184}$ \\
         & Acc. & $94.84\pm0.45$ & $91.46\pm0.75$ & $93.06\pm0.32$ & $91.86\pm0.22$ & $94.86\pm0.76$ & $92.86\pm0.48$ & $\mathbf{95.88\pm0.10}$ \\
         & F1 & $94.91\pm0.42$ & $91.61\pm0.75$ & $93.18\pm0.33$ & $92.00\pm0.23$ & $94.94\pm0.75$ & $92.98\pm0.49$ & $\mathbf{95.94\pm0.09}$ \\\midrule
         & MAE & $\mathbf{0.116\pm0.006}$ & $0.133\pm0.005$ & $0.136\pm0.009$ & $0.128\pm0.002$ & $0.119\pm0.006$ & $0.120\pm0.002$ & $\mathbf{0.115\pm0.006}$ \\
        BS & MSE & $0.219\pm0.017$ & $0.219\pm0.024$ & $0.220\pm0.017$ & $0.208\pm0.008$ & $0.223\pm0.019$ & $\mathbf{0.199\pm0.009}$ & $0.213\pm0.009$\\
         & Acc. & $90.78\pm0.95$ & $90.22\pm0.85$ & $91.20\pm0.42$ & $89.24\pm0.37$ & ${91.62\pm0.39}$ & $89.56\pm0.14$ & $\mathbf{91.74\pm1.22}$\\ 
         & F1 & $86.20\pm1.63$ & $85.65\pm1.08$ & $87.00\pm0.77$ & $84.25\pm0.62$ & ${87.67\pm0.66}$ & $84.73\pm0.15$ & $\mathbf{87.84\pm1.89}$ \\[3pt]
        \midrule
        & MAE & 2 & 6 & 5.75 & 5,75 & 2.75 & 4.24 & \textbf{1.25} \\
        Avg. Rank & MSE & 3.75 & 4 & 5.25 & 5.75 & 4.25 & 3 & \textbf{1.75} \\
        & Acc. & 4 & 5.25 & 4.25 & 5.75 & 2.75 & 4.5 & \textbf{1.5} \\
        & F1 & 4 & 5.25 & 4.25 & 5.75 & 2.75 & 4.5 & \textbf{1.5} \\
        \bottomrule
    \end{tabular}
    \caption{Performances of different activation functions on formulae. Mean absolute error (MAE) and mean squared error (MSE) are calculated after target being normalized by standard scaling. A classification task is also defined for each expression, whose target is the bin index. Accuracy and F1 scores are calculated for classification tasks. To summarize the performances, we also calculate the average rank in each metric, with higher rank meaning better performance.}
    \label{tab:math}
\end{table*}

\newdimen\R 
\R=3.5cm 
\newdimen\L 
\L=4cm

\newcommand{\Dmt}{8} 
\newcommand{\Umt}{13} 

\newcommand{\Amt}{360/\Dmt} 
\begin{figure}[h!t]
    \centering
\begin{tikzpicture}[scale=0.6]
  \path (0:0cm) coordinate (O); 

  \foreach \X in {1,...,\Dmt}{
    \draw (\X*\Amt:0) -- (\X*\Amt:\R);
  }

  \pgfmathsetmacro{\UU}{\Umt-3}
  \foreach \Y in {0,...,\UU}{
    \pgfmathsetmacro{\Z}{\Y+3};
    \foreach \X in {1,...,\Dmt}{
      \path (\X*\Amt:\Z*\R/\Umt) coordinate (D\X-\Y);
      \fill (D\X-\Y) circle (1pt);
    }
    \draw [opacity=0.3] (0:\Z*\R/\Umt) \foreach \X in {1,...,\Dmt}{
        -- (\X*\Amt:\Z*\R/\Umt)
    } -- cycle;
  }

  \path (1*\Amt:\L) node (L1) {\small GS(reg)};
  \path (2*\Amt:\L) node (L2) {\small AR(reg)};
  \path (3*\Amt:\L) node (L3) {\small NS(reg)};
  \path (4*\Amt:\L) node (L4) {\small BS(reg)};
  \path (5*\Amt:\L) node (L5) {\small GS(clf)};
  \path (6*\Amt:\L) node (L6) {\small AR(clf)};
  \path (7*\Amt:\L) node (L7) {\small NS(clf)};
  \path (8*\Amt:\L) node (L8) {\small BS(clf)};

  %
  %

  \draw [color=red,line width=1.5pt,opacity=1.0]
    (D1-8.993) -- (D2-7.103) --
    (D3-9.946) -- (D4-5.595) --
    (D5-6.329) -- (D6-3.211) --
    (D7-7.634) -- (D8-5.796) --
    cycle;

  \draw [color=green,line width=1.5pt,opacity=1.0]
    (D1-6.248) -- (D2-3.971) --
    (D3-0.964) -- (D4-1.548) --
    (D5-6.438) -- (D6-1.845) --
    (D7-0) -- (D8-3.910) --
    cycle;

  \draw [color=blue,line width=1.5pt,opacity=1.0]
    (D1-3.753) -- (D2-5.507) --
    (D3-1.384) -- (D4-0.625) --
    (D5-0) -- (D6-4.680) --
    (D7-3.623) -- (D8-7.750) --
    cycle;

  \draw [color=orange,line width=1.5pt,opacity=1.0]
    (D1-0) -- (D2-2.75) --
    (D3-0.446) -- (D4-5.030) --
    (D5-7.641) -- (D6-0) --
    (D7-0.903) -- (D8-0) --
    cycle;

  \draw [color=pink,line width=1.5pt,opacity=1.0]
    (D1-9.664) -- (D2-6.338) --
    (D3-8.563) -- (D4-4.048) --
    (D5-5.562) -- (D6-10) --
    (D7-7.691) -- (D8-3.166) --
    cycle;

  \draw [color=purple,line width=1.5pt,opacity=1.0]
    (D1-2.627) -- (D2-9.375) --
    (D3-4.580) -- (D4-8.810) --
    (D5-10) -- (D6-0.727) --
    (D7-3.166) -- (D8-1.309) --
    cycle;

  \draw [color=brown,line width=1.5pt,opacity=1.0]
    (D1-10) -- (D2-9.853) --
    (D3-9.554) -- (D4-7.083) --
    (D5-8.559) -- (D6-5.610) --
    (D7-10) -- (D8-10) --
    cycle;

  \begin{scope}[shift={(5,0)}]
    \draw [color=red,line width=1.5pt,opacity=1.0] (0,3.5) -- (1,3.5) node[right] {ReLU};
    \draw [color=green,line width=1.5pt,opacity=1.0] (0,2.5) -- (1,2.5) node[right] {ELU};
    \draw [color=blue,line width=1.5pt,opacity=1.0] (0,1.5) -- (1,1.5) node[right] {SELU};
    \draw [color=orange,line width=1.5pt,opacity=1.0] (0,0.5) -- (1,0.5) node[right] {Swish};
    \draw [color=pink,line width=1.5pt,opacity=1.0] (0,-0.5) -- (1,-0.5) node[right] {NLReLU};
    \draw [color=purple,line width=1.5pt,opacity=1.0] (0,-1.5) -- (1,-1.5) node[right] {GELU};
    \draw [color=brown,line width=1.5pt,opacity=1.0] (0,-2.5) -- (1,-2.5) node[right] {CombU};
  \end{scope}

\end{tikzpicture}

    \caption{Radar graph of performances of each different activation functions on different mathematical expressions. ``clf'' stands for the classification task, and ``reg'' stands for the regression task. The values are normalized based on the range of the metric on the dataset over different activation functions and then rounded to 0.1. Value 0 is not at the center, but at some distance out of the center, for better visibility. For each experiment, 
    the average of the mean scores after normalization in the experiments is drawn as data points in the graph.
    }
    \label{fig:radar-math}
\end{figure}
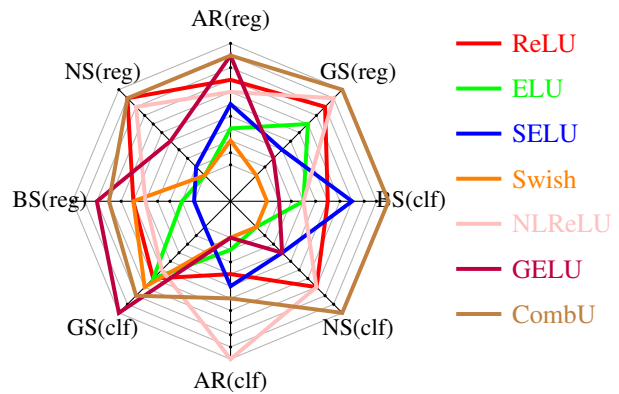

Experimental results are summarized in Table~\ref{tab:math}, and a more visual-friendly radar graph is drawn in Figure~\ref{fig:radar-math}. Although CombU is not always the best for all metrics in all experiments, the result shows a clear stable advantage of CombU over other activation functions. CombU performs the best on more than half of the experiments. On those experiments where CombU does not perform the best, it is also the second or third among the seven tested activation functions, and there is no steady winner that is always the best among the rest activation functions.

\subsection{Fitting Real-world Datasets}

The advantage of CombU also extends to real-world datasets, which is unsurprising since real-world tabular datasets are likely to embed some implicit mathematical relations. In this section, we show the results of a series of tabular classification and regression tasks on common benchmark datasets, using MLPs of different activation functions. We selected 4 classification and 4 regression datasets:

\begin{itemize}
    \item \textbf{Classification}: \begin{itemize}
        \item Breast Cancer Wisconsin (Original) Data Set (BW)~\cite{breast-w-data}
        \item Diabetes (DI)~\cite{diabetes-data}
        \item Breast Cancer Wisconsin (Diagnostic) Data Set (WDBC) (WD)~\cite{wdbc-data}
        \item Phishing websites (PW)~\cite{pweb-data}
    \end{itemize}
    \item \textbf{Regression}: \begin{itemize}
        \item AutoMPG (AM)~\cite{autompg-data}
        \item Boston housing (BS)~\cite{boston-data}
        \item California housing (CH)~\cite{cahouse-data}
        \item Wine quality (WQ)~\cite{wineq-data}
    \end{itemize}
\end{itemize}

More details about the datasets are shown in Table~\ref{tab:datasets}. Each dataset can be obtained from library \verb|sklearn.datasets.fetch_openml(HANDLE)|.

\begin{table*}[h!t]
    \centering
    \begin{tabular}{cccccc}
        \toprule
        Abbreviation & Handle & \#R & \#N & \#C & Task \\
        \midrule
        CH~\cite{cahouse-data} & california\_housing & 20640 & 8 & 1 & reg\\
        DI~\cite{diabetes-data} & diabetes & 768 & 8 & 0 & clf(2)\\
        BS~\cite{boston-data} & boston & 506 & 13 & 2 & reg \\
        BW~\cite{breast-w-data} & breast-w & 699 & 9 & 0 & clf(2) \\
        WQ~\cite{wineq-data} & wine\_quality & 6497 & 11 & 0 & reg\\
        WD~\cite{wdbc-data} & wdbc & 569 & 30 & 0 & clf(2) \\
        PW~\cite{pweb-data} & PhishingWebsites & 11055 & 0 & 30 & clf(2)\\
        AM~\cite{autompg-data} & autoMpg & 398 & 4 & 3 & reg\\
        \bottomrule
    \end{tabular}
    \caption{Summary of all real-world datasets used. Handle refers to the OpenML handle, as all datasets are fetched from OpenML using $\mathtt{scikit-learn}$. \#R refers to the number of rows, \#N refers the number of numerical features, \#C refers to the number of classification features. In the Task column, ``reg'' stands for regression, ``clf'' stands for classification, with a subsequent bracket indicating the number of different classes.}
    \label{tab:datasets}
\end{table*}

For all datasets' features, numerical values are normalized by standard scaling, and categorical values are one-hot encoded.
Regression tasks' targets are also standard scaled, based on which the reported MSE and MAE scores are calculated. All datasets are randomly split into train and test sets by $4:1$ ratio. 
For all experiments, we take 5 runs of the same experiment and collect the scores.

\begin{table*}[h!t]
    \small
    \setlength{\tabcolsep}{3pt}
    \centering
    \begin{tabular}{ccccccccc}
        \toprule
        \multicolumn{2}{c}{Dataset} & ReLU & ELU & SELU & Swish & NLReLU & GELU & CombU\\
        \midrule
        BW & Acc. & $95.14\pm0.29$ & $\mathbf{97.14\pm0.00}$ & $97.00\pm0.29$ & $95.71\pm0.00$ & $95.29\pm0.03$ & $95.43\pm0.35$ &$95.43\pm0.03$\\
        & F1 & $94.52\pm0.31$ & $\mathbf{96.80\pm0.00}$ & $96.64\pm0.31$ & $95.16\pm0.03$ & $94.70\pm0.41$ & $95.83\pm0.38$ & $94.86\pm0.41$\\\midrule
        DI& Acc. & $74.29\pm2.04$ & $74.42\pm0.66$ & $74.15\pm1.26$ & $75.19\pm0.07$ & $74.04\pm1.40$ & $\mathbf{75.71\pm0.78}$ & $\mathbf{75.71\pm3.05}$ \\
         & F1 & $70.18\pm2.18$ & $69.11\pm0.73$ & $69.47\pm1.66$ & $70.19\pm1.20$ & $71.24\pm1.97$ & $71.46\pm1.12$ & $\mathbf{71.80\pm3.05}$ \\\midrule
        WD & Acc. & $95.96\pm0.43$ & $96.14\pm0.43$ & $95.95\pm0.89$ & $\mathbf{97.37\pm0.00}$ & $96.49\pm0.55$ & $95.96\pm0.43$ & $96.67\pm0.35$\\
         & F1 & $95.54\pm0.47$ & $95.77\pm0.46$ & $95.59\pm0.95$ & $\mathbf{97.07\pm0.00}$ & $96.12\pm0.61$ & $95.59\pm0.46$ & $96.33\pm0.38$\\\midrule
        PW & Acc. & $97.58\pm0.11$ & $97.27\pm0.21$ & $97.13\pm0.22$ & $96.37\pm0.27$ & $\mathbf{97.60\pm0.17}$ & $97.13\pm0.22$ & $97.46\pm0.07$\\
        & F1 & $97.54\pm0.11$ & $97.22\pm0.21$ & $97.09\pm0.22$ & $96.31\pm0.27$ & $\mathbf{97.56\pm0.17}$ & $96.87\pm0.33$ & $97.41\pm0.07$\\
        \midrule
        Avg. Rank & Acc. & 4.75 & 3.25 & 4.5 & 4 & 3.5 & 4.25 & \textbf{2.5} \\
        & F1  & 5.25 & 4 & 4.5 & 4 & 4 & 3.25 & \textbf{2.75} \\
        \bottomrule
    \end{tabular}
    \caption{Performances of different activation functions on classification datasets. Metrics collected are accuracy and F1 scores in \%. The average accuracy and F1 scores are calculated after scores are normalized for different datasets. 
    }
    \label{tab:clf}
\end{table*}

\begin{table*}[h!t]
    \small
    \setlength{\tabcolsep}{3pt}
    \centering
    \begin{tabular}{ccccccccc}
        \toprule
        \multicolumn{2}{c}{Dataset} & ReLU & ELU & SELU & Swish & NLReLU & GELU & CombU\\
        \midrule
        AM & MAE & $0.222\pm0.006$ & $0.242\pm0.002$ & $0.248\pm0.006$ & $0.229\pm0.004$ & $\mathbf{0.219\pm0.004}$ & $0.222\pm0.003$ & $0.223\pm0.006$ \\
        & MSE & $0.090\pm0.004$ & $0.095\pm0.004$ & $0.099\pm0.004$ & $0.086\pm0.004$ & $0.089\pm0.004$ & $\mathbf{0.084\pm0.005}$ & $0.091\pm0.006$ \\\midrule
        BS & MAE & $0.232\pm0.007$ & $0.268\pm0.002$ & $0.263\pm0.006$ & $0.258\pm0.004$ & $0.235\pm0.010$ & $0.253\pm0.001$ & $\mathbf{0.228\pm0.003}$\\
        & MSE & $0.093\pm0.005$ & $0.123\pm0.002$ & $0.122\pm0.004$ & $0.114\pm0.003$ & $0.095\pm0.006$ & $0.109\pm0.003$ & $\mathbf{0.090\pm0.003}$\\\midrule
        CH & MAE & $0.289\pm0.005$ & $0.298\pm0.008$ & $0.298\pm0.005$ & $0.300\pm0.005$ & $0.290\pm0.004$ & $0.295\pm0.007$ & $\mathbf{0.283\pm0.004}$\\
        & MSE & $0.204\pm0.006$ & $0.200\pm0.003$ & $0.199\pm0.003$ & $0.204\pm0.003$ & $0.201\pm0.002$ & $0.199\pm0.004$ & $\mathbf{0.192\pm0.003}$\\\midrule
        WQ & MAE & $0.532\pm0.004$ & $0.605\pm0.002$ & $0.617\pm0.006$ & $0.616\pm0.004$ & $\mathbf{0.528\pm0.004}$ & $0.610\pm0.005$ & $0.544\pm0.006$\\\
         & MSE & $0.583\pm0.012$ & $0.597\pm0.004$ & $0.599\pm0.008$ & $0.696\pm0.004$ & $\mathbf{0.567\pm0.006}$ & $0.612\pm0.009$ & $0.573\pm0.011$\\
        \midrule  
        Avg. Rank & MAE & \textbf{2} & 5.5 & 6.25 & 5.75 & \textbf{2} & 3.75 & 2.25 \\
         & MSE & 3.75 & 5.25 & 5 & 4.75 & 3 & 3.5 & \textbf{2.25} \\
        \bottomrule
    \end{tabular}
    \caption{Performances of different activation functions on regression datasets. Metrics collected are mean absolute error (MAE) and mean squared error (MSE), after target being normalized by standard scaling. 
    }
    \label{tab:reg}
\end{table*}


Classification results can be found in Table~\ref{tab:clf}, and regression results can be found in Table~\ref{tab:reg}. A consolidated view of performances on all different datasets can be found in Figure~\ref{fig:radar}.
Given that each dataset has its own unique statistics and properties, it is difficult to determine a universally optimal activation function. Therefore, we compared the overall performance across a series of experiments. By calculating the mean rank within each metric, CombU demonstrated a significant advantage in both classification and regression tasks.
Additionally, CombU is the only activation function that consistently ranks within the top three across all metrics, indicating its relatively stable and superior performance.

\newcommand{\Dds}{16} 
\newcommand{\Uds}{13} 

\newcommand{\Ads}{360/\Dds} 
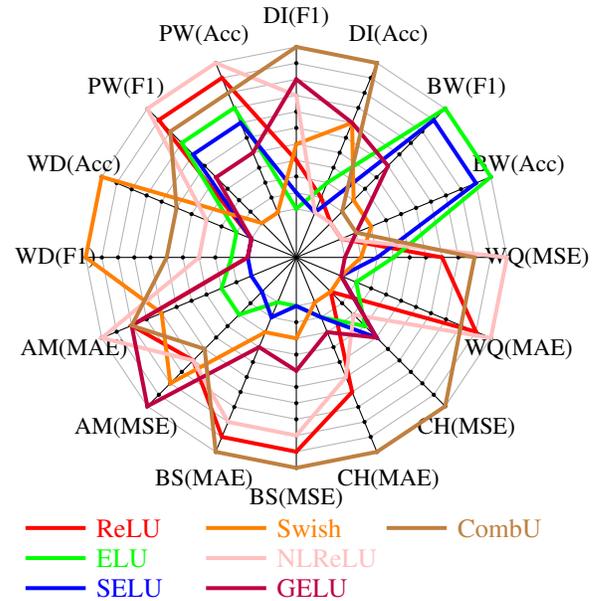
\begin{figure}
    \centering
\begin{tikzpicture}[scale=0.8]
  \path (0:0cm) coordinate (O); 

  \foreach \X in {1,...,\Dds}{
    \draw (\X*\Ads:0) -- (\X*\Ads:\R);
  }

  \pgfmathsetmacro{\UdsU}{\Uds-3}
  \foreach \Y in {0,...,\UdsU}{
    \pgfmathsetmacro{\Z}{\Y+3};
    \foreach \X in {1,...,\Dds}{
      \path (\X*\Ads:\Z*\R/\Uds) coordinate (D\X-\Y);
      \fill (D\X-\Y) circle (1pt);
    }
    \draw [opacity=0.3] (0:\Z*\R/\Uds) \foreach \X in {1,...,\Dds}{
        -- (\X*\Ads:\Z*\R/\Uds)
    } -- cycle;
  }

  \path (1*\Ads:\L) node (L1) {\small BW(Acc)};
  \path (2*\Ads:\L) node (L2) {\small BW(F1)};
  \path (3*\Ads:\L) node (L3) {\small DI(Acc)};
  \path (4*\Ads:\L) node (L4) {\small DI(F1)};
  \path (5*\Ads:\L) node (L5) {\small PW(Acc)};
  \path (6*\Ads:\L) node (L6) {\small PW(F1)};
  \path (7*\Ads:\L) node (L7) {\small WD(Acc)};
  \path (8*\Ads:\L) node (L8) {\small WD(F1)};
  \path (9*\Ads:\L) node (L9) {\small AM(MAE)};
  \path (10*\Ads:\L) node (L10) {\small AM(MSE)};
  \path (11*\Ads:\L) node (L11) {\small BS(MAE)};
  \path (12*\Ads:\L) node (L12) {\small BS(MSE)};
  \path (13*\Ads:\L) node (L13) {\small CH(MAE)};
  \path (14*\Ads:\L) node (L14) {\small CH(MSE)};
  \path (15*\Ads:\L) node (L15) {\small WQ(MAE)};
  \path (16*\Ads:\L) node (L16) {\small WQ(MSE)};

  %
  %

  \draw [color=red,line width=1.5pt,opacity=1.0]
    (D1-0) -- (D2-0) --
    (D3-1.497) -- (D4-3.978) --
    (D5-9.837) -- (D6-9.84) --
    (D7-0) -- (D8-0) --
    (D9-8.966) -- (D10-6) --
    (D11-9) -- (D12-9.091) --
    (D13-6.471) -- (D14-0) --
    (D15-9.551) -- (D16-6.444) --
    cycle;

  \draw [color=green,line width=1.5pt,opacity=1.0]
    (D1-10) -- (D2-10) --
    (D3-2.275) -- (D4-0) --
    (D5-7.317) -- (D6-7.28) --
    (D7-1.277) -- (D8-1.484) --
    (D9-2.069) -- (D10-2.667) --
    (D11-0) -- (D12-0) --
    (D13-1.176) -- (D14-3.333) --
    (D15-1.348) -- (D16-3.333) --
    cycle;

  \draw [color=blue,line width=1.5pt,opacity=1.0]
    (D1-9.3) -- (D2-9.298) --
    (D3-0.659) -- (D4-1.338) --
    (D5-6.179) -- (D6-6.24) --
    (D7-0) -- (D8-0.323) --
    (D9-0) -- (D10-0) --
    (D11-1.25) -- (D12-0.303) --
    (D13-1.176) -- (D14-4.167) --
    (D15-0) -- (D16-2.889) --
    cycle;

  \draw [color=orange,line width=1.5pt,opacity=1.0]
    (D1-2.85) -- (D2-2.807) --
    (D3-6.887) -- (D4-4.015) --
    (D5-0) -- (D6-0) --
    (D7-10) -- (D8-10) --
    (D9-6.552) -- (D10-8.667) --
    (D11-2.5) -- (D12-2.727) --
    (D13-0) -- (D14-0) --
    (D15-0.112) -- (D16-1.333) --
    cycle;

  \draw [color=pink,line width=1.5pt,opacity=1.0]
    (D1-0.75) -- (D2-0.789) --
    (D3-0) -- (D4-7.918) --
    (D5-10) -- (D6-10) --
    (D7-3.759) -- (D8-3.742) --
    (D9-10) -- (D10-6.667) --
    (D11-8.25) -- (D12-8.485) --
    (D13-5.882) -- (D14-2.5) --
    (D15-10) -- (D16-10) --
    cycle;

  \draw [color=purple,line width=1.5pt,opacity=1.0]
    (D1-1.45) -- (D2-5.746) --
    (D3-6.886) -- (D4-8.736) --
    (D5-4.472) -- (D6-4.48) --
    (D7-0) -- (D8-0.324) --
    (D9-8.966) -- (D10-10) --
    (D11-3.75) -- (D12-4.242) --
    (D13-2.941) -- (D14-4.167) --
    (D15-0.787) -- (D16-0) --
    cycle;

  \draw [color=brown,line width=1.5pt,opacity=1.0]
    (D1-1.45) -- (D2-1.491) --
    (D3-10) -- (D4-10) --
    (D5-8.862) -- (D6-8.8) --
    (D7-5.035) -- (D8-5.097) --
    (D9-8.621) -- (D10-5.333) --
    (D11-10) -- (D12-10) --
    (D13-10) -- (D14-10) --
    (D15-8.202) -- (D16-8.667) --
    cycle;

  \begin{scope}[shift={(-4.5,-6)}]
    \draw [color=red,line width=1.5pt,opacity=1.0] (0,1.5) -- (1,1.5) node[right] {ReLU};
    \draw [color=green,line width=1.5pt,opacity=1.0] (0,1) -- (1,1) node[right] {ELU};
    \draw [color=blue,line width=1.5pt,opacity=1.0] (0,0.5) -- (1,0.5) node[right] {SELU};
    \draw [color=orange,line width=1.5pt,opacity=1.0] (3,1.5) -- (4,1.5) node[right] {Swish};
    \draw [color=pink,line width=1.5pt,opacity=1.0] (3,1) -- (4,1) node[right] {NLReLU};
    \draw [color=purple,line width=1.5pt,opacity=1.0] (3,0.5) -- (4,0.5) node[right] {GELU};
    \draw [color=brown,line width=1.5pt,opacity=1.0] (6,1.5) -- (7,1.5) node[right] {CombU};
  \end{scope}

\end{tikzpicture}

    \caption{Radar graph of performances of each different activation functions on different datasets. The values are normalized based on the range of the metric on the dataset over different activation functions and then rounded to 0.1. Value 0 is not at the center, but at some distance out of the center, for better visibility.}
    \label{fig:radar}
\end{figure}

\subsection{Generative Tasks on Tabular Data}

\begin{table*}[h!t]
    \small
    \setlength{\tabcolsep}{2pt}
    \centering
    \begin{tabular}{cccccccccc}
        \toprule
        \multicolumn{2}{c}{Dataset} & Real & ReLU & ELU & SELU & Swish & NLReLU & GELU & CombU\\
        \midrule
        & CTGAN MAE & 0.267 & $0.521\pm0.008$ & $0.676\pm0.009$ & $0.684\pm0.005$ & $0.725\pm0.018$ & $0.587\pm0.025$ & $0.650\pm0.012$ & $\mathbf{0.513\pm0.006}$ \\
        CH & CTGAN MSE & 0.162 & $0.504\pm0.005$ & $0.771\pm0.007$ & $0.762\pm0.021$ & $0.881\pm0.046$ & $0.613\pm0.053$ & $0.799\pm0.020$ & $\mathbf{0.471\pm0.009}$ \\
         & TVAE MAE & 0.267 & $0.471\pm0.006$ & $0.485\pm0.007$ & $0.496\pm0.009$ & $0.509\pm0.003$ & $0.491\pm0.001$ & $0.478\pm0.003$ & $\mathbf{0.460\pm0.006}$ \\
        & TVAE MSE & 0.162 & $0.425\pm0.014$ & $0.439\pm0.013$ & $0.437\pm0.020$ & $0.528\pm0.006$ & $0.454\pm0.007$ & $0.463\pm0.005$ & $\mathbf{0.399\pm0.006}$\\
        \midrule
        & CTGAN Acc. & 97.11 & $83.21\pm0.76$ & $79.65\pm0.46$ & $84.38\pm0.49$ & $83.24\pm0.65$ & $76.38\pm1.91$ & $84.37\pm0.56$ & $\mathbf{84.47\pm0.22}$\\
        PW & CTGAN F1 & 97.07 & $83.17\pm0.75$ & $79.59\pm0.45$ & ${84.21\pm0.48}$ & $82.93\pm0.65$ & $76.29\pm1.94$ & $83.98\pm0.55$ & $\mathbf{84.24\pm0.18}$\\
        & TVAE Acc. & 97.11 & $\mathbf{89.46\pm0.17}$ & $88.84\pm0.43$ & $88.69\pm0.35$ & $88.92\pm0.32$ & $87.43\pm0.54$ & $89.27\pm0.65$ & ${89.40\pm0.13}$\\
        & TVAE F1 & 97.07 & ${89.23\pm0.18}$ & $88.72\pm0.44$ & $88.52\pm0.36$ & $88.74\pm0.35$ & $87.40\pm0.54$ & $89.17\pm0.67$ & $\mathbf{89.28\pm0.13}$ \\
        \bottomrule
    \end{tabular}
    \caption{Performances of different activation functions on tabular generation tasks. Metrics collected are machine learning efficacy, with accuracy (Acc., in \%) and F1 (in \%) score for classification and mean absolute error (MAE) and mean squared error (MSE) for regression after target being normalized by standard scaling. Column ``Real'' means the performance of the real data as training set. 
    }
    \label{tab:gen}
\end{table*}

By the same intuition of implicit mathematical expressions in real-world datasets, we try different tasks on the same datasets. One suitable task would be generative task. In this section, we try different activation functions at all layers for two major generative models with MLP backbone: CTGAN~\cite{ctgan} for GAN~\cite{gan}, and TVAE~\cite{ctgan} for VAE~\cite{vae}, which are also two of the most well-known tabular generative models. We will run experiments on real-world datasets with more than 10,000 rows only because generative tasks typically are more demanding on the data. Evaluation is done by machine learning efficacy, calculated by training a machine learning model on the synthetic data as the training set and testing on a real test set not seen by the generative model. The result is shown in Table~\ref{tab:gen}. Each generative model is trained to generate three sets of samples with the same size of the real training set to obtain the scores. CombU shows a significant advantage over the rest of activation functions for generative tasks.

\section{Discussion}


No single activation function consistently outperforms all others across all datasets and metrics. Consequently, we must compare activation functions based on their general performance and the stability of achieving strong results. Our experiments confirmed that CombU exhibits superior average performance and greater stability in achieving upper-half rankings compared to other activation functions, particularly in fitting mathematical expressions.

Although CombU theoretically has the capability to perfectly fit mathematical expressions (with the exception of the Black-Scholes formula, which involves the cumulative density function of the standard normal distribution so it lacks a closed-form expression), none of our experiments demonstrated perfect performance by CombU on mathematical expressions. Possible reasons for this include the following:

\begin{itemize}
    \item Training of MLPs are done by gradient descent or its variants, which is never guaranteed to be a perfect solver, but an approximation of the target. Therefore, although CombU activation allows the MLPs to perfectly fit the target, they are very likely to be trapped in local optima.
    \item The mathematical proof of the capability of exponential and logarithm, and hence polynomial (which relies on the proof of the former), depends on the constraint that the training dataset is finite and thus has a fixed range in terms of the maximum and minimum non-zero absolute values. There is chance that the test set feature values, or intermediate values, are out of the range of training set. Also, the dependence on the empirical range, especially the minimum non-zero absolute values, may make the model less numerically stable.
    \item For some expressions, the MLP size may not be sufficient to perfectly express the formulae. 
\end{itemize}

The advantage of CombU in typical supervised tasks is less obvious and stable than mathematical expressions. This is likely because of the confounders, which make the target variable not purely expressed by mathematical expressions but with an error term. Confounders make it nearly impossible to represent the target variable with a perfect mathematical expression, thereby limiting CombU's advantage. Additionally, confounders increase the risk of overfitting, as the network may attempt to fit the error terms—terms that cannot theoretically be inferred from the features.

The advantage of CombU is more obvious in generative tasks than supervised tasks. This is likely because of the simpler but pervasive implicit mathematical relations that describe the interaction between different features. 

In addition, the experiments' standard deviations show CombU's advantage in producing stable results, especially in generative tasks.

\section{Conclusion \& Future Works}

\subsection{Conclusion}

In this paper, we propose a new activation function -- CombU, which effectively combines activations with linear, exponential, and logarithmic components. We have theoretically proven that by integrating different activation functions within the same layer across multiple layers, neural networks can achieve near-perfect fitting of most mathematical expressions. Our experiments further demonstrate that CombU outperforms conventional activation functions, which are applied uniformly across all layers and dimensions, on prediction tasks of mathematical formulae, classification and regression datasets, and on generation tasks for tabular data.


\subsection{Future Work}

So far, CombU is implemented by a simple \textit{for} loop that applies different activation functions sequentially within the same activation layer. However, we anticipate that applying these activation functions in parallel could accelerate model training and inference, allowing us to fully leverage the available computing resources.


Other mathematical expressions, such as those involving trigonometric functions, may also be represented perfectly by a CombU-activated neural network if the neural network is complex-valued.
This is supported by Euler's formula $e^{ix}=\cos x+i\sin x$, which establishes a relationship between exponential and trigonometric functions.

\bibliography{ref}
\bibliographystyle{aaai}

\appendix
\section{Appendix A: Proof of Basic Rules}
\label{app:proof}

Notations and constraints in this appendix would stay aligned with Section~\nameref{sec:theo:proof}.

\begin{lemma}
    \label{thm:lc-next}
    \textbf{Law of Linear Combinations in Next Layer.} Any linear combination (with bias) of the output (after activation) of a previous layer can be the input of the next layer (before activation). Formally writing, $\forall l\in\{0,1,\dots,L-1\},j\in\{1,2,\dots,D_{l+1}\},x_j^{[l+1]}=(\boldsymbol{\theta}^{[l+1]}_{\cdot j})^T\mathbf{z}^{[l]}+\theta^{[l+1]}_{0j}$, which is a linear combination of $\mathbf{z}^{[l]}$.
\end{lemma}

Lemma~\ref{thm:lc-next}'s correctness is obvious by design of MLP neural networks.

\begin{thm}
    \textbf{Law of Layer Persistence.} For any value as the output of a certain layer (after activation), it can be constructed as the input of any subsequent layer. Formally writing, $\forall l'>l,i\in\{1,2,\dots,D_l\},j\in\{1,2,\dots,D_{l'}\},\exists\boldsymbol{\Theta}$ such that $z^{[l]}_i=x^{[l']}_j$.
\end{thm}

\begin{proof}
    If $l'=l+1$, there is no activation between $z^{[l]}_i$ and $x^{[l']}_j$, and $z^{[l]}_i$ is clearly a linear combination of $\mathbf{z}^{[l]}$. Thus, this case is proven.

    If $l'>l+1$, suppose $\forall i\in\{1,2,\dots,D_l\},j'\in\{1,2,\dots,D_{l'-1}\}, z^{[l]}_i=x^{[l'-1]}_{j'}$ can be constructed by some $\boldsymbol{\Theta}$, involving parameters $\boldsymbol{\theta}^{[l'']}$ where $l''<l$. Then $-z^{[l]}_i=x^{[l'-1]}_{j''}$ can also be constructed by $\boldsymbol{\theta}^{[l'-1]}_{\cdot j''}=-\boldsymbol{\theta}^{[l'-1]}_{\cdot j'}$ with some $j''\ne j'$, which must exist with sufficiently large $D_{l'-1}$. Since $j',j''$ are arbitrary, they can fall on dimensions with whichever activation function wanted. Here we suppose the activation functions at layer $l'-1$ are ReLU at position $j'$ and $j''$. Then, let

    \begin{equation}
        \label{eq:per-construct}
        \boldsymbol{\theta}^{[l']}_{ij}=\begin{cases}
            0 & \text{ if }i\notin\{j',j''\}\\
            1 & \text{ if }i=j'\\
            -1 & \text{ if }i=j''
        \end{cases}
    \end{equation}
    
    Then we have 

    \begin{align}
        x^{[l']}_j&=(\boldsymbol{\theta}^{[l']}_{\cdot j})^T\cdot \mathbf{z}^{[l'-1]}\\
        &=z^{[l'-1]}_{j'}-z^{[l'-1]}_{j''} & \text{by Equation~\ref{eq:per-construct}}\\
        &=R(x^{[l'-1]}_{j'})-R(x^{[l'-1]}_{j'}) \\
        &=R(z^{[l]}_i)-R(-z^{[l]}_i)& \text{by definition of }j', j''\\
        &=\begin{cases}
            z_i^{[l]}-0&\text{ if }z_i^{[l]}\ge0\\
            0-(-z_i^{[l]})&\text{ if }z_i^{[l]}<0
        \end{cases}\\
        &=z^{[l]}_i
    \end{align}

    Thus, by induction, the claim is proven.
\end{proof}

\begin{lemma}
    \textbf{Law of Redundant Input Dimensions.} If a function can be constructed by a network, then the same function taking some redundant input dimensions can be constructed by a network. Formally*h writing, if $f(\mathbf{x})$ with $\mathbf{x}\in\mathbb{R}^D$ can be constructed by a network, then $f'(\mathbf{x}')$ with $\mathbf{x}'\in\mathbb{R}^{D'}$ where $D'\ge D$ and $f'(\mathbf{x}')=f(\mathbf{x}'_{1..D})$ can be constructed by a network. 
\end{lemma}

\begin{proof}
    The network for $f'$ can be directly adapted from the network for $f$, where all other layers' weights and biases remain the same except for the first layer, where additional first layer dimensions, and hence weights $\boldsymbol{\theta}^{[1]}_{(D+1..D')\cdot}$, are needed. Setting these additional weights to 0, then second layer's input, $\mathbf{x}^{[1]}$ can be maintained the same as the network for $f$. Thus, $f'(\mathbf{x}')=f(\mathbf{x}'_{1..D})$ can be easily constructed.
\end{proof}

\begin{thm}
    \textbf{Law of Composition.} If all values defined by $\mathbf{f}(\mathbf{x})=[f_1(\mathbf{x}),f_2(\mathbf{x}),\dots,f_{m}(\mathbf{x}')]$, for $m$ different functions on $\mathbf{x}\in\mathbb{R}^D$, can be constructed by the network, and $f(\mathbf{x}')$ with $\mathbf{x}'\in\mathbb{R}^m$ can be constructed by the network, then $f(\mathbf{f}(\mathbf{x}))$ can be constructed by the network, regardless of the layer and dimensions all functions can be constructed.
\end{thm}

\begin{proof}
    Let $l$ be the maximum of minimum layer where $f_1,f_2,\dots,f_{m}$ can be constructed. Then, by Law of Layer Persistence, at layer $l$, $\mathbf{f}(\mathbf{x})$ can be constructed as part of its dimensions. Without loss of generality, let these dimensions be dimensions $1,2,\dots,m$. We will use standard notations on this network.

    Then, let $\boldsymbol{\theta}^{[l]}_{\cdot(m+1..2m)}=-\boldsymbol{\theta}^{[l]}_{\cdot(1..m)}$, we have $-\mathbf{f}(\mathbf{x})$ constructed at the same layer at dimensions $m+1,m+2,\dots,2m$. We can also suppose that all these dimensions are assigned the activation ReLU without loss of generality. Thus,

    \begin{equation}
        z^{[l]}_i=\begin{cases}
            f_i(\mathbf{x})&\text{ if }i\in\{1,2,\dots,m\}\\
            -f_{i-m}(\mathbf{x})&\text{ if }i\in\{m+1,m+2,\dots,2m\}\\
            0&\text{ if }i>2m
        \end{cases}
    \end{equation}

    In the network for $f$, on which we will use notations with $\cdot'$ to represent, every value in the first layer input, ${\mathbf{x}'}^{[1]}$, is a linear combination of input of $f$, $\mathbf{x}'$. To construct the composition, we need $\mathbf{x}'=\mathbf{f}(\mathbf{x})$. In other words, every value in ${\mathbf{x}'}^{[1]}$ is a linear combination of $\mathbf{f}(\mathbf{x})$. Without loss of generality, take an arbitrary value in ${\mathbf{x}'}^{[1]}$, and suppose it is $x=\boldsymbol{\theta}\cdot\mathbf{f}(\mathbf{x})+b$, then 

    \begin{align}
        x&=\boldsymbol{\theta}\cdot\mathbf{f}(\mathbf{x})+b\\
        &=\begin{pmatrix}
            \boldsymbol{\theta}\\-\boldsymbol{\theta}\\\mathbf{0}
        \end{pmatrix}\cdot\begin{pmatrix}
            R(\mathbf{f}(\mathbf{x}))\\R(-\mathbf{f}(\mathbf{x}))\\\mathbf{0}
        \end{pmatrix}+b\\
        &=\begin{pmatrix}
            \boldsymbol{\theta}\\-\boldsymbol{\theta}\\\mathbf{0}
        \end{pmatrix}\cdot\mathbf{z}^{[l]}+b
    \end{align}

    which means $x$ is also a linear combination of $\mathbf{z}^{[l]}$. Thus, ${\mathbf{x}'}^{[1]}$ is a valid $\mathbf{x}^{[l+1]}$ (after padding zeros).

    Then, by concatenating the network for $f_1,f_2,\dots,f_m$, $\begin{pmatrix}
        \boldsymbol{\theta}\\-\boldsymbol{\theta}\\\mathbf{0}
    \end{pmatrix}$, and network for $f$, and by Law of Redundant Input Dimensions to adjust the dimension alignment, $f(\mathbf{f}(\mathbf{x}))$ can be constructed by a network.
\end{proof}

\section{Appendix B: Experiment Setup}
\label{app:exp}

\subsection{Mathematical Expressions}

For all expressions, we generate 5000 samples, with 80\% as training data, and the rest 20\% as test data. Each sample's features are randomly generated based on some na\"ive rules, and target values are rigorously calculated based on the formulae. The purpose of generating these datasets are experiments to verify the effect of activation functions, so the feature values may not necessarily be the common value range in reality. The rules are randomly hand-crafted for giving sufficient variability in the input pattern.

In the following subsections, we use the notations in Table~\ref{tab:math-notation}.

\begin{table}[h!t]
    \centering
    \begin{tabular}{cp{0.75\linewidth}}
        \toprule
        Notation & \multicolumn{1}{c}{Description} \\
        \midrule
        $\mathcal{N}(\mu,\sigma)$ & Normal distribution with mean $\mu$ and standard deviation $\sigma$.\\
        $\mathcal{U}(a,b)$ & Uniform distribution between $[a,b]$. \\
        $\mathcal{V}(a.b)$ & Uniform distribution between $[a,b]$ but integers only. \\
        $\mathcal{D}(\mathbf{v},\mathbf{p})$ & Discrete distribution on values $\mathbf{v}$ on based on densities $\mathbf{p}$, where $\|\mathbf{p}\|_1=1$ and $\forall i,p_i\in[0,1]$. \\
        $\mathcal{D}_1^{\mathcal{D}_2}$ & Given two distributions $\mathcal{D}_1,\mathcal{D}_2$, value $v=a\times10^b$ where $a\sim\mathcal{D}_1,b\sim\mathcal{D}_2$. \\
        \bottomrule
    \end{tabular}
    \caption{Notation used in mathematical expression definitions.}
    \label{tab:math-notation}
\end{table}

\subsubsection{Gaussian Distribution Expression (GS)}

Gaussian distribution's~\cite{stats} is one of the most essential and well-known distributions in basic statistics theory, and is used widely in various different domains for quantitative analysis. Its probability density function is 

\begin{equation}
    f(x)=\frac{1}{\sigma\sqrt{2\pi}}e^{-\frac{1}{2}(\frac{x-\mu}{\sigma})^2}
\end{equation}

where $\mu$ is the mean and $\sigma$ is the standard deviation, contains some mathematical expressions mentioned previously. Moreover, the calculation of the mean and standard deviation using a set of sampled values also involve some basic calculation. Thus, we can construct a target variable for $f(x)$.

To be more specific, we construct independent variables as described in Table~\ref{tab:normal-features}.

\begin{table}[h!t]
    \centering
    \begin{tabular}{cp{0.65\linewidth}}
        \toprule
        Variable(s) & \multicolumn{1}{c}{Description} \\
        \midrule
        $x_1,x_2,\dots,x_8$ & Based on a randomly generated mean value $\mu\sim\mathcal{N}(0,10)$ and standard value $\sigma\sim\mathcal{U}(1,6)$, sample $x_1,x_2,\dots,x_8\sim\mathcal{N}(\mu,\sigma)$. \\
        $v$ & Based on the actual sample mean $\overline{x}$ and standard deviation $\sigma_x$ of $x_1,x_2,\dots,x_8$, sample $v\sim\mathcal{N}(\overline{x},\sigma_x)$. \\
        \bottomrule
    \end{tabular}
    \caption{Gaussian distribution expression's feature construction.}
    \label{tab:normal-features}
\end{table}

Then, based on the sample mean $\overline{x}$ and standard deviation $\sigma_x$, calculate $f(v)$ as the target variable to be predicted.

\subsubsection{Modified Arrhenius Equation (AR)}

Arrhenius equation~\cite{ar-base} is an important formula in physical chemistry about the relation between temperature and reaction rates of chemical reactions, which is given as

\begin{equation}
    k=Ae^{-\frac{E_a}{RT}}
\end{equation}

where $k$ is the rate constant, $T$ is the absolute temperature, $A$ is the pre-exponential factor (also known as Arrhenius factor or frequency factor), $E_a$ is the molar activation energy, and $R\approx8.314J\cdot mol^{-1}\cdot K^{-1}$ is the universal gas constant.

To make explicit the temperature dependence of the pre-exponential factor, the modified Arrhenius equation~\cite{ar-math} is given as follows:

\begin{equation}
    k=AT^ne^{-\frac{E_a}{RT}}
\end{equation}

where $n$ describes this relation. 

We take the left hand side of the equation as the target to predict, and right hand side as the expression with independent variables. Detailed construction of the independent variables is given in Table~\ref{tab:arr-features}.

\begin{table}[h!t]
    \centering
    \begin{tabular}{cp{0.7\linewidth}}
        \toprule
        Variable & Description \\
        \midrule
        $n$ & Randomly generated $n\sim V(0,10)$. \\
        $T$ & Randomly generated $T\sim\mathcal{U}(1,11)$. \\
        $E_a$ & Randomly generated $E_a\sim\mathcal{U}(0,100)$. \\
        $A$ & Randomly generated $A\sim\mathcal{U}(0,1)^{\mathcal{V}(-2, 1)}$. \\
        \bottomrule
    \end{tabular}
    \caption{Modified Arrhenius equation's feature construction.}
    \label{tab:arr-features}
\end{table}

\subsubsection{3D Steady-State Vortex Solution of Naiver-Stokes Equation (NS)}

Naiver-Stokes equations~\cite{dyn-base} are important equations in fluid dynamics describing the motion of viscous fluid substances that expresses momentum balance of Newtonian fluids under conservation of mass.  These equations are a set of partial differential equations, which is not suitable for our experiment purpose. However, its solutions under certain special circumstances can be constructed by complex but suitable mathematical expressions. 

A steady-state vortex solution in 3D for non-viscous gas, whose density, velocities, and pressure goes to zero far from the origin, on the velocity is given as follows~\cite{ns-math}.

\begin{align}
    \mathbf{u}(x,y,z)&=\frac{A}{(r^2+x^2+y^2+z^2)^2}\begin{pmatrix}
        2(-ry+xz)\\2(rx+yz)\\r^2-x^2-y^2+z^2
    \end{pmatrix}
\end{align}

where $x,y,z$ are the positions in 3D coordinate, $r$ is the constant radius of the inner coil, $\mathbf{u}$ is the flow velocity, and $A$ is an arbitrary constant.

We calculate the Euclidean norm of the flow velocity of the solution, which is $\|\mathbf{u}\|$, as the output target. In other words, the mathematical expression to be predicted is 

\begin{multline}
    \|\mathbf{u}\|=\frac{A}{(r^2+x^2+y^2+z^2)^2}\cdot\\\sqrt{(2(-ry+xz))^2+(2(rx+yz))^2+(r^2-x^2-y^2+z^2)^2}
\end{multline}

Note that the solution is actually a solution for a differential function, where $A$ is an arbitrary constant, but existence of arbitrary values in the expression makes deterministic prediction for the purpose of our experiment impossible. Thus, we fix their values, so that they are also considered independent variables. Details of how these values are sampled is given in Table~\ref{tab:ns-features}.

\begin{table}[h!t]
    \centering
    \begin{tabular}{cp{0.7\linewidth}}
        \toprule
        Variable(s) & Description \\
        \midrule
        $A$ & Randomly generated $A\sim U(0,1)$. \\
        $r,x,y,z$ & Randomly generated $r,x,y,z\sim\mathcal{U}(1,10)^{\mathcal{D}([-3,-2,-1,0,1],[.1,.2,.4,.2,.1])}$. \\
        \bottomrule
    \end{tabular}
    \caption{Naiver-Stokes equation steady-state vortex solution's feature construction.}
    \label{tab:ns-features}
\end{table}

\subsubsection{Black-Scholes Formula (BS)}

Black-Scholes model, also known as Black-Scholes-Merton model~\cite{bs-math}, is a well-known mathematical model for a financial market in estimating prices. Its foundation is Black-Scholes Formula, which calculates the price of a put option based on put-call parity with discount factor $e^{-r(T-t)}$:

\begin{equation}
    P(S_t,t)=Ke^{-r(T-t)}-S_t+C(S_t,t)
\end{equation}

where $t$ is the time, $T$ is the time of option expiration, $S_t$ is the price of the underlying asset at time $t$, $P(S_t,t)$ is the price of an European put option, $K$ is the strike price, also known as the the exercise price, and $C(S_t,t)$ is the value of an European call option, which is calculated by

\begin{align}
    C(S_t,t)&=N\left(\frac{\ln(\frac{S_t}{K})+(r+\frac{\sigma^2}{2})(T-t)}{\sigma\sqrt{T-t}}\right)S_t\\
    &\quad-N\left(\frac{\ln(\frac{S_t}{K})+(r-\frac{\sigma^2}{2})(T-t)}{\sigma\sqrt{T-t}}\right)Ke^{-r(T-t)}
\end{align}

where $N(\cdot)$ denotes the standard normal cumulative distribution function

\begin{equation}
    N(x)=\frac{1}{\sqrt{2\pi}}\int_{-\infty}^xe^{-\frac{z^2}{2}}dz
\end{equation}

which has no closed form integral result.

Details of feature construction for the Black-Scholes formula expression can be found in Table~\ref{tab:bs-features}.

\begin{table}[h!t]
    \centering
    \begin{tabular}{cp{0.7\linewidth}}
        \toprule
        Variable & Description \\
        \midrule
        $\sigma$ & Randomly generated $\sigma\sim U(0,100)$. \\
        $T-t$ & Randomly generated $(T-t)\sim\mathcal{V}(1,9)^{\mathcal{V}(1,3)}$. \\
        $S_t$ & Randomly generated $S_t\sim\mathcal{U}(1,10)^{\mathcal{V}(0,4)}$. \\
        $K$ & Randomly generated $K\sim\mathcal{U}(1,10)^{\mathcal{V}(0,4)}$, where the exponential part of the distribution, instead of being independent of, is the same value as $S_t$. \\
        $r$ & Randomly generated $r\sim U(0,0.1)$. \\
        \bottomrule
    \end{tabular}
    \caption{Black-Scholes formula's feature construction.}
    \label{tab:bs-features}
\end{table}

\subsection{Model and Training Setup}

All experiments are run with batch size 500 and learning rate $5\times10^{-4}$ with Adam optimizer. MLPs are implemented with dropout rate 0.1.

For all mathematical expressions, 200 epochs are trained.

Note that the sizes of all datasets mentioned in Table~\ref{tab:datasets} vary a lot, and it makes sense to vary the model sizes accordingly. We use two MLP settings: a smaller one and a larger one, where the smaller one is used for datasets with less than 1000 rows, and the larger one is used for datasets with more than 1000 rows. The smaller MLP consists of 3 layers with 128 hidden dimensions, and the larger MLP consists of 6 layers with 256 hidden dimensions. All MLPs are implemented with a dropout rate 0.1.

Smaller models are trained for 500 epochs, and larger models are trained for 200 epochs. 

\section{Appendix C: Supplementary Experiment Results}
\label{app:sup}

\subsection{Loss and Performance Trend on Training Set}

\begin{figure*}[h!t]
    \centering
    \includegraphics[width=\linewidth]{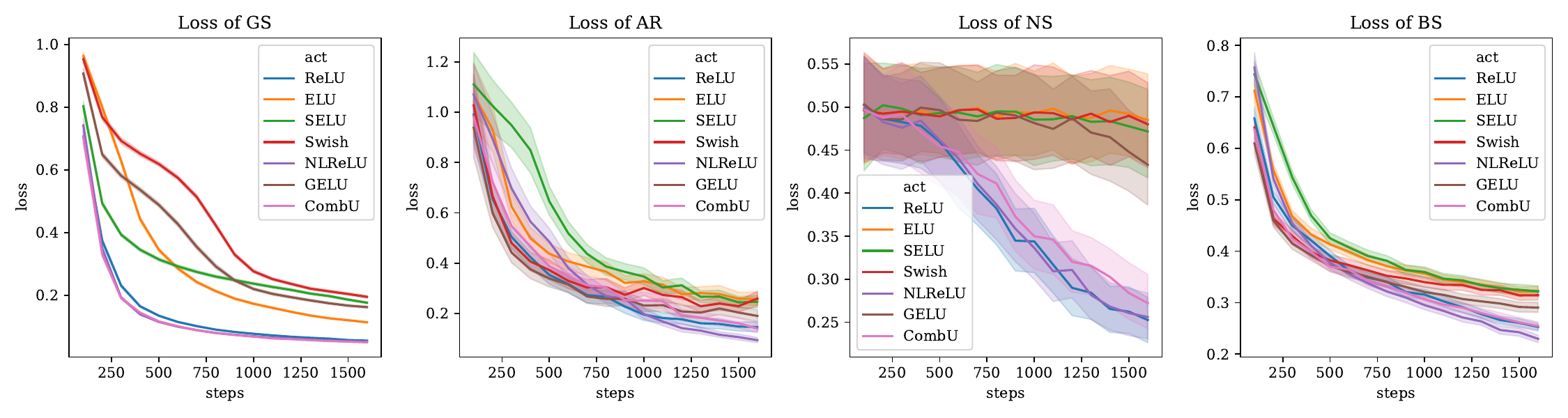}
    \caption{Training loss trend of different activation functions on mathematical expression regression experiments. The training losses are MSE losses, and average over each 100 steps of all 5 runs are drawn. The first two epochs are skipped to make the difference of later trend easier to see.}
    \label{fig:formula-reg-loss}
\end{figure*}

\begin{figure*}[h!t]
    \centering
    \includegraphics[width=\linewidth]{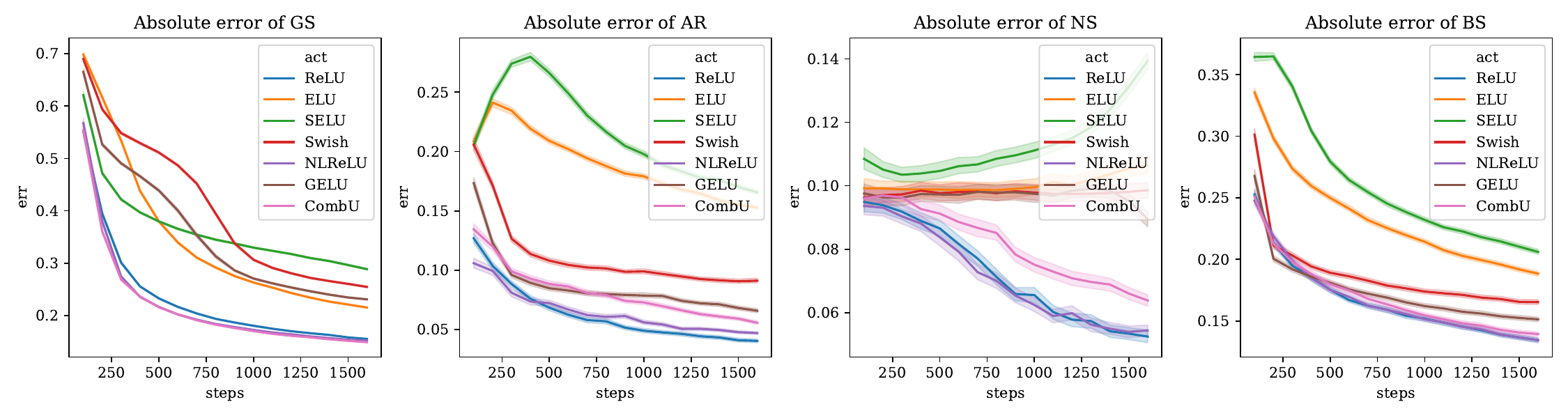}
    \caption{Training absolute error trend of different activation functions on mathematical expression regression experiments. Average over each 100 steps of all 5 runs are drawn. The first two epochs are skipped to make the difference of later trend easier to see.}
    \label{fig:formula-reg-err}
\end{figure*}

\begin{figure*}[h!t]
    \centering
    \includegraphics[width=\linewidth]{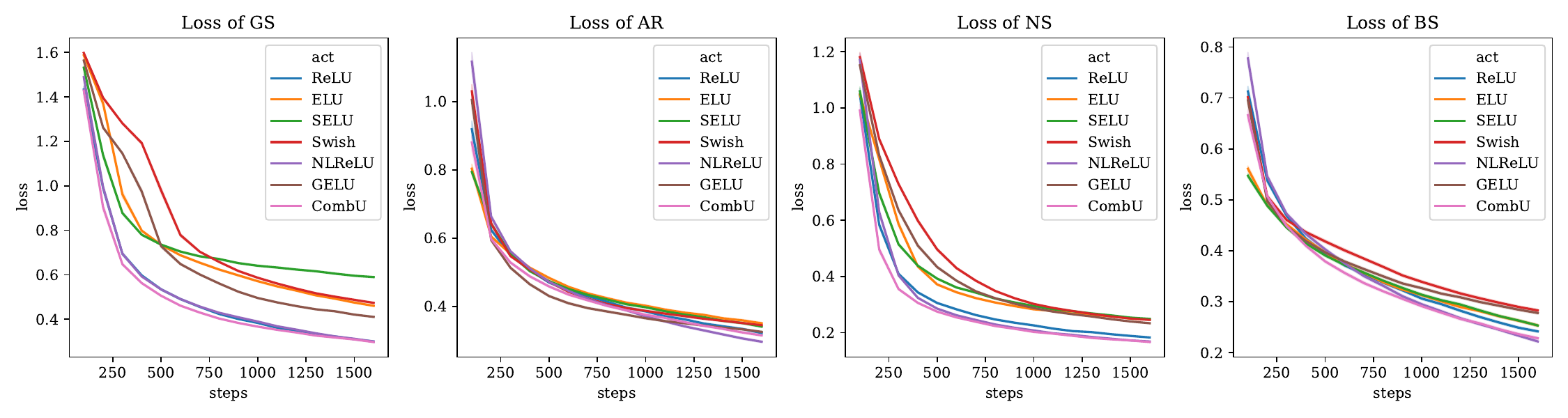}
    \caption{Training absolute error trend of different activation functions on mathematical expression classification experiments. The training losses are cross entropy, and average over each 100 steps of all 5 runs are drawn. The first two epochs are skipped to make the difference of later trend easier to see.}
    \label{fig:formula-clf-loss}
\end{figure*}

\begin{figure*}[h!t]
    \centering
    \includegraphics[width=\linewidth]{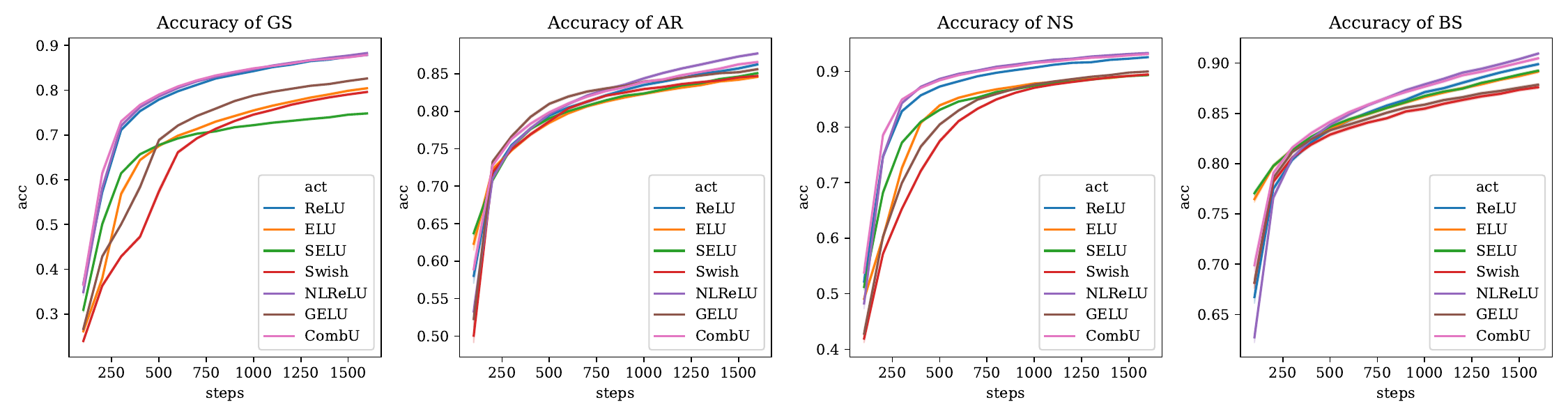}
    \caption{Training accuracy trend of different activation functions on mathematical expression classification experiments. Average over each 100 steps of all 5 runs are drawn. The first two epochs are skipped to make the difference of later trend easier to see.}
    \label{fig:formula-clf-acc}
\end{figure*}

The trends of losses, absolute error, and accuracy during training process for mathematical expressions are shown in Figure~\ref{fig:formula-reg-loss}-\ref{fig:formula-clf-acc}. The loss trends indicates that all activation functions at least makes the MLPs valid, so that losses are gradually decreasing. Also, there is no fixed pattern of training performances, for all different experiments, but CombU is usually among the best, and never appear as one of the worst.

\subsection{Experiments on Image Classification}

We also tried CombU on other deep neural networks to see its effect. A natural intuitive of such experiments is that however complex the task is, there must be some mathematical relation embedded in the relation of the features to the target.

\begin{table*}[h!t]
    \small
    \setlength{\tabcolsep}{2pt}
    \centering
    \begin{tabular}{ccccccccc}
        \toprule
        Dataset & Model & ReLU & ELU & SELU & Swish & NLReLU & GELU & CombU\\
        \midrule
         & CNN & $67.50\pm0.20$ & $\mathbf{70.98\pm0.65}$ & $\mathbf{70.77\pm0.16}$ & $68.74\pm0.56$ & $65.90\pm0.82$ & $70.11\pm0.20$ & $68.31\pm0.31$\\
        CIFAR10 & ResNet18 & $\mathbf{91.71\pm0.20}$ & $91.28\pm0.19$ & $90.46\pm0.24$ & $91.61\pm0.30$ & $90.43\pm0.40$ & $\mathbf{91.77\pm0.28}$ & $91.23\pm0.15$\\
         & ResNet101 & $91.40\pm0.11$ & $91.24\pm0.98$ & $90.79\pm0.37$ & $\mathbf{91.94\pm0.38}$ & $48.10\pm2.74$ & $\mathbf{91.94\pm0.07}$ & $91.08\pm0.45$\\
        \midrule
         & CNN & $34.34\pm0.28$ & $37.83\pm0.58$ & $\mathbf{38.97\pm0.53}$ & $35.22\pm0.60$ & $27.01\pm0.36$ & $35.31\pm0.21$ & $33.63\pm0.67$\\
        CIFAR100 & ResNet18 & $69.01\pm0.21$ & $69.66\pm0.12$ & $68.44\pm0.31$ & $\mathbf{70.34\pm0.30}$ & $65.96\pm0.56$ & $69.98\pm0.36$ & $69.13\pm0.37$\\
         & ResNet101 & $69.63\pm0.47$ & $70.46\pm0.10$ & $70.15\pm0.24$ & $\mathbf{71.43\pm0.71}$ & $10.91\pm1.17$ & $70.14\pm0.09$ & $67.55\pm0.75$\\
        \bottomrule 
    \end{tabular}
    \caption{Performances of different activation functions on CIFAR classification datasets. Metrics collected are accuracy in \%. Model CNN refers to the simple CNN model. WA stands for the weighted average of the (mean) performance over the activation functions involved in CombU.}
    \label{tab:cifar}
\end{table*}

We run CIFAR10 and CIFAR100 classification tasks~\cite{cifar} on simple CNN as given in the PyTorch image classification tutorial\footnote{The tutorial from \url{https://pytorch.org/tutorials/beginner/blitz/cifar10_tutorial.html}.}, and ResNet18 and ResNet101~\cite{resnet} by 100, 80, and 50 epochs each, with learning rate $2\times10^{-4}$, weight decay $1\times10^{-5}$ and batch size 128. The result of shown in Table~\ref{tab:cifar}. The results are collected from 3 runs.

One obvious conclusion unrelated to combinations is that NLReLU obviously does not fit for these tasks, as it is the worst in all cases, and is significantly bad in ResNet101. A possible reason for this is that the input values to the activation values are much larger than tabular MLPs, making the gradient very small since it is the gradient of a logarithmic expression. 

The results show that CombU does not outperform other activation functions. We suspect three major reasons for this outcome: complex target, deep network, and unstructured data features. 

\begin{itemize}
    \item \textbf{Bad Performance for Single Activation.} The NLReLU is having very bad performance overall, which makes the combined activation performing bad.
    \item \textbf{Complex target.} When the explicit mathematical expressions involved in these complex tasks, if they exist, get too complex, either too deeply nested that the number of layers of the network is insufficient, or involving some expressions not captured by the current CombU, such as trigonometric or some integrals without closed form, CombU may not be able to show its advantage by fitting mathematical expressions. 
    \item \textbf{Deep network.} When the network gets more complex and deep, whether the mathematical relation is found becomes less important, as whichever activation function would eventually lead to some piece-wise relation from features to target, and when the network is complex enough, the number of pieces is large enough to estimate the relation. 
    \item \textbf{Unstructured data features.} For less structured data features (for example, images and text), typically fewer outliers are found in the data, and all features are likely less skewed than raw tabular data. Consequently, the importance of finding a perfect mathematical expression eclipses, since piece-wise estimations would do the job equally well empirically.
\end{itemize}



\section{Appendix D: Code Implementation of CombU}

The code of the implementation of CombU is quite simple and straight-forward, so instead of putting it on GitHub, we directly put the code here.

\begin{lstlisting}[language=Python,basicstyle=\scriptsize]
import numpy as np
import torch
from typing import *
from torch import nn

# Register activation functions
SimpleActivation = Literal["relu", "elu", "selu", ...]
ACTIVATIONS = {
  "relu": nn.ReLU,
  "elu": nn.ELU,
  "selu": nn.SELU,
  ...
}

def make_simple(
    act_type: SimpleActivation, *args, **kwargs
) -> nn.Module:
  act = ACTIVATIONS[act_type]
  return act(*args, **kwargs)

class CombU(nn.Module):
  def __init__(
    self, ratio: Dict[SimpleActivation, float],
    dim: Union[int, Sequence[int]],
    act_args: Dict[SimpleActivation, Dict]):
    """
    Args:
      ratio: ratio of each activation, sum to 1.
      dim: shape of the input. 
      act_args: additional args for each activation.
    """
    super().__init__()
    self.ratio = ratio

    # Prepare activations
    self.activations = nn.ModuleDict({
      k: make_simple(k, **act_args.get(k, {})) 
      for k in ratio
    })

    # Calculate dimensions
    base_dim = dim if isinstance(dim, int) else dim[0]
    dims = {
      k: round(v * base_dim) for k, v in ratio.items()
    }
    diff = base_dim - sum(dims.values())
    for k, v in dims.items():
      if diff == 0:
        break
      elif diff > 0:
        dims[k] = v + diff
        diff = 0
      else:
        change = min(v, -diff)
        dims[k] = v - change
        diff += change

    # Assign dimensions
    self.dims = {}
    masked = np.zeros(dim, dtype=np.bool_)
    for k, v in dims.items():
      m = np.random.choice(
        np.arange(base_dim)[~masked],
        v, replace=False
      )
      self.dims[k] = m
      masked[m] = True

  def forward(self, x: torch.Tensor):
    activated_x = torch.empty_like(x)
    for k, v in self.dims.items():
      ax = self.activations[k](x[:, v])
      activated_x[:, v] = ax
    return activated_x
\end{lstlisting}

Note that unlike typical activation functions, CombU need to know the dimension in advance. Also, if one wants to save the weights and load from some checkpoints to resume on other tasks, it is important to save and load the assigned dimensions or the random seed at the start of constructing the activation layer, otherwise the network make produce unexpected results.

In the case of multi-dimension input to activation layers (e.g. activation layers between CNN layers), we consider the size of the first dimension (e.g. channel of images) as $D$ and the activation function is applied uniformly on all components in the same dimension. Although flattening all dimensions as $D$ would also work in theory, it does not make sense for CNN. This is because the same network weights are applied to different input dimensions for CNN, but if different dimensions are activated using different functions, the convolution layer's weight would thus become meaningless. Therefore, we unify the activation function on the same channel.

\end{document}